\documentclass{article}
\usepackage{iclr2027_conference,times}

\usepackage{amsmath,amsfonts,bm}
\usepackage{algorithm}
\usepackage{algorithmic}
\usepackage{amssymb}
\usepackage{amsmath}
\usepackage{booktabs}
\usepackage{multirow}
\usepackage{pifont}
\usepackage{colortbl}
\usepackage{graphicx}
\usepackage{amsthm}
\usepackage{subcaption}
\usepackage{nicematrix}
\newtheorem{proposition}{Proposition}[section]
\usepackage{wrapfig}
\usepackage{listings}

\def\eqref#1{equation~\ref{#1}}
\def\Eqref#1{Equation~\ref{#1}}
\def\1{\bm{1}}

\DeclareMathAlphabet{\mathsfit}{\encodingdefault}{\sfdefault}{m}{sl}
\SetMathAlphabet{\mathsfit}{bold}{\encodingdefault}{\sfdefault}{bx}{n}

\DeclareMathOperator*{\argmax}{arg\,max}

\usepackage{graphicx}
\usepackage{booktabs}
\usepackage{multirow}
\usepackage{hyperref}
\usepackage{url}
\usepackage{xspace}
\hypersetup{hidelinks}
\usepackage{marvosym}

\title{
ReCAP: \hspace{-0.05em}Retrieval-Guided
\hspace{-0.05em}Capability
\hspace{-0.05em}Reuse
\hspace{-0.05em}for
\hspace{-0.03em}Multimodal
\hspace{-0.03em}Continual
\hspace{-0.03em}Instruction
\hspace{-0.03em}Tuning
}

\author{
{Tao Hu$^{1,2}$  \quad
Zhinuo Zhou$^{3}$    \quad
Xialiang Tong$^{3}$    \quad
De-Chuan Zhan$^{1,2}$    \quad
Da-Wei Zhou$^{1,2}\textsuperscript{(\Letter)}$} \\
$^{1} $ School of Artificial Intelligence, Nanjing University \\
$^{2} $ State Key Laboratory for Novel Software Technology, Nanjing University \\
$^{3} $ Noah’s Ark Lab, Huawei Technologies \\
\texttt{\small \{hut, zhandc, zhoudw\}@lamda.nju.edu.cn}\\
\texttt{\small \{zhouzhinuo1, tongxialiang\}@huawei.com}
}

\newcommand{\method}{\textsc{ReCAP}\xspace}

\iclrfinalcopy
\begin{document}

\maketitle
\fancyhead[L]{Preprint.}

\begin{abstract} 
Multimodal continual instruction tuning (MCIT) aims to enable multimodal large language models to acquire new capabilities from sequential tasks while preserving previously learned knowledge. Existing methods primarily mitigate catastrophic forgetting by constraining parameter updates or separating task-specific adaptations. However, continual adaptation can also benefit from external knowledge that provides domain-specific information and reusable reasoning patterns for solving diverse instructions. For example, to answer ``How many red cubes are to the left of the sphere?'', domain knowledge can provide relevant concepts about objects and spatial relations, while reasoning knowledge can specify ordered operations such as object recognition, spatial filtering, and counting. Despite this potential, how to leverage external knowledge for continual adaptation remains largely unexplored in existing MCIT methods. To this end, we propose \method, a retrieval-guided framework that leverages external knowledge to guide capability reuse during continual adaptation. At each continual stage, \method uses external search and an LLM to incrementally build a knowledge base of domain, reasoning, and format knowledge based on the current-stage training data. For each instruction, retrieved domain knowledge guides generation, while retrieved reasoning knowledge selects and orders capability modules to form an instance-specific capability path. As these capability modules are reused across stages, subsequent adaptation can overwrite previously learned parameters. To enable stable cross-stage reuse, \method introduces adaptive subspace recycling, which parameterizes reusable capability modules with shared bases and stage-specific cores, protects historically important directions while recycling residual capacity.
Extensive experiments on MCIT benchmarks show that \method achieves SOTA performance.\looseness -1

\end{abstract}
\section{Introduction}

Multimodal large language models (MLLMs)~\citep{liu2023visual,dai2023instructblip} acquire broad vision-language capabilities through instruction tuning~\citep{zhang2026instruction,tong2025metamorph}, yet real-world deployment is inherently dynamic. New visual domains, question types, and reasoning requirements continuously emerge after initial training, while retraining on all previously observed data is often unavailable or prohibitively expensive. Multimodal continual instruction tuning (MCIT) therefore aims to enable models to acquire new capabilities from sequential tasks while preserving previously learned knowledge~\citep{chen2024coin}. Existing MCIT methods primarily mitigate catastrophic forgetting by constraining parameter updates or introducing task-specific adaptations, including prompt-based methods~\citep{zeng2025modalprompt}, low-rank modules~\citep{chen2024coin}, expandable components~\citep{guo2025hidellava}, and expert-based routing~\citep{huai2025clmoe,xie2026same}. While effective in reducing cross-task interference, these methods largely rely on model parameters and task-specific adaptation modules to encode and preserve what is learned across continual stages.

However, continual adaptation can also benefit from external knowledge. A natural way to access such knowledge is retrieval-augmented generation (RAG)~\citep{guu2020realm,lewis2020retrieval,asai2024selfrag}, which connects models to an editable external memory and allows knowledge beyond model parameters to be independently updated and retrieved. In MCIT, such knowledge can provide more than factual information, particularly for handling diverse and evolving task requirements. 
\begin{wrapfigure}{r}{0.3\textwidth}
    % \vspace{-3mm}
    \centering
    \includegraphics[width=\linewidth]{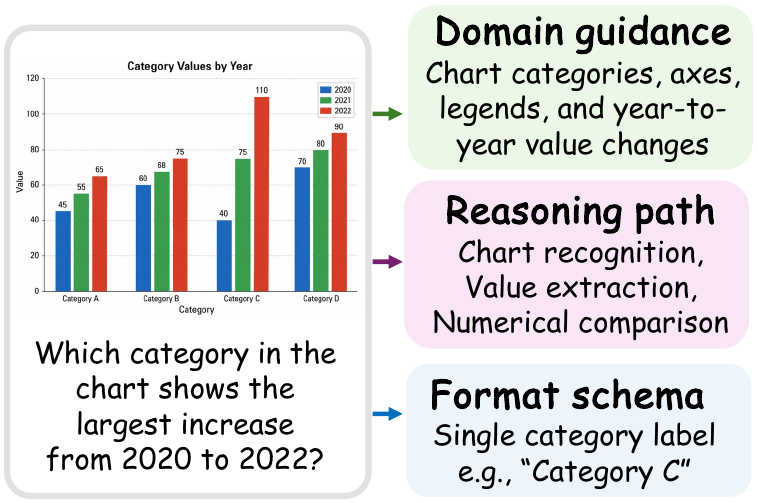}
    \vspace{-7mm}
    \caption{\small An example requiring chart
    understanding and comparison reasoning.}
    \vspace{-4mm}
    \label{fig:example}
\end{wrapfigure}
Domain knowledge can supply relevant concepts, relations, and visual context for the current multimodal instruction, while reasoning knowledge can describe the ordered operations needed to solve it. For example, as illustrated in Figure~\ref{fig:example}, answering ``Which category in the chart shows the largest increase from 2020 to 2022?'' may require knowledge of chart conventions, axes, and legends, together with a reasoning process involving visual element recognition, value extraction, and comparison. Output conventions can further characterize the expected response form. Despite these benefits, existing MCIT methods have largely underexplored how external knowledge can support continual adaptation.

The key challenge, therefore, is how to make such external knowledge actionable for continual adaptation. Reasoning knowledge is particularly useful in this regard, as its ordered operations naturally reveal the capabilities required to solve an instruction. Vanilla RAG~\citep{lewis2020retrieval} primarily conditions generation on retrieved information, but does not explicitly organize these operations into reusable capability units. Representing these operations as parameterized capability modules provides an explicit interface through which retrieval can select and compose reusable capabilities for each instruction, while allowing recurring capabilities to be reused across continual stages. However, such parameterized reuse introduces a new continual-learning challenge: adapting a recurring capability at later stages may overwrite previously learned behaviors. Effective knowledge-guided MCIT therefore requires both retrieval-guided capability composition and an adaptive mechanism for preserving reusable capabilities during continual adaptation.

Motivated by these considerations, we propose \method, a retrieval-guided framework for capability reuse in MCIT. At each continual stage, \method uses the current-stage training data, external search, and an LLM to incrementally build a structured knowledge base of domain, reasoning, and format cards. For each multimodal instruction, retrieved domain knowledge provides prompt guidance for generation, while retrieved reasoning knowledge identifies the required capability operations and organizes the corresponding parameterized modules into an instance-specific path. The selected modules are then composed sequentially along this path during the forward pass. Format knowledge provides response-form cues that further support reasoning retrieval and inference-time core routing. To support stable reuse of recurring capabilities across continual stages, \method introduces adaptive subspace recycling, which represents them with shared bases and stage-specific cores. Directions carrying high accumulated historical energy are protected from subsequent interference, while low-energy directions are recycled and new directions are appended to provide capacity for further adaptation. Finally, to enable task-free use of the retained stage-specific cores, retrieval-guided routing selects an appropriate core for each activated capability at inference.

In summary, our contributions are threefold: (1) we introduce retrieval-guided capability reuse for MCIT, where external knowledge provides generation guidance and organizes instance-level capability composition and reuse; (2) we propose adaptive subspace recycling for stable cross-stage capability reuse through shared bases, stage-specific cores, and protected historical directions; and (3) experiments on MCIT benchmarks demonstrate \method's effectiveness over representative baselines.\looseness -1

\section{Related Work}

\textbf{Multimodal continual instruction tuning:}
MCIT extends continual learning from fixed-label prediction to unified
vision--language instruction following. CoIN introduced a representative
benchmark and mixture-of-expert LoRA adaptation for sequential multimodal tasks
~\citep{chen2024coin}. Subsequent methods explore different strategies to
balance stability and plasticity, including modality-aware prompting
~\citep{zeng2025modalprompt}, expert specialization and routing
~\citep{huai2025clmoe}, expandable task-specific components with transferable
representations~\citep{guo2025hidellava}, and joint stabilization of routing
and parameter updates~\citep{xie2026same}. Orthogonal low-rank updates provide
another parameter-isolation strategy for reducing interference between
continual updates~\citep{wang2023orthogonal}. Meanwhile, LoRA-based adaptation
~\citep{hu2022lora} has become a common approach for storing compact task
adaptations. These methods mainly focus on parameter-space
solutions, either separating adaptations across tasks or constraining updates
to preserve previous knowledge.

\textbf{Retrieval-augmented generation:}
Retrieval-augmented generation combines a parametric model with an external, editable memory, enabling models to access and update knowledge beyond their parameters. Existing studies explore retrieval-enhanced pretraining~\citep{borgeaud2022improving}, knowledge-intensive generation, multi-passage fusion~\citep{izacard2021leveraging}, adaptation over external corpora~\citep{izacard2022atlas}, query transformation~\citep{gao2023hyde}, and retrieval-aware generation evaluation~\citep{asai2024selfrag}. These methods primarily focus on improving knowledge access, factual grounding, and generation quality. In contrast, the role of retrieval as a mechanism for organizing continual capability acquisition and reuse in MCIT remains largely unexplored.

\section{Preliminaries}
\label{sec:preliminaries}

\subsection{Multimodal Continual Instruction Tuning}

Multimodal continual instruction tuning (MCIT) considers a sequence of
multimodal instruction tuning tasks. Specifically, the training stream consists
of \(T\) datasets:
\begin{equation}
    \mathcal{D}_{1:T}=\{\mathcal{D}_t\}_{t=1}^{T}, \qquad
    \mathcal{D}_t=\{(\mathbf{v}_i^t,\mathbf{q}_i^t,
    \mathbf{y}_i^t)\}_{i=1}^{n_t},
    \label{eq:stream}
\end{equation}
where \(\mathbf{v}\), \(\mathbf{q}\), and \(\mathbf{y}\) denote an image,
an instruction, and its target response, respectively. At stage \(t\), the
learner observes only the current dataset \(\mathcal{D}_t\) and sequentially
updates the model. After learning stage \(t\), the model is evaluated on all
observed tasks \(\mathcal{D}_{1:t}\), requiring it to acquire new capabilities
while preserving previously learned knowledge.

\subsection{Baselines in Multimodal Continual Instruction Tuning}

To mitigate catastrophic forgetting, existing MCIT methods mainly follow two
typical strategies: constraining parameter updates or separating task-specific
adaptations.
\\\noindent\textbf{Constraining Parameter Updates:}
Continually updating shared adaptation parameters may interfere with previously
learned knowledge. Methods such as O-LoRA~\citep{wang2023orthogonal} and
SAME~\citep{xie2026same} therefore restrict current-stage updates using
information retained from previous stages. The constrained update is
represented as:
\begin{equation}
\Delta\boldsymbol{\theta}_{t}^{\mathrm{c}}
=
\mathcal{C}
\left(
\Delta\boldsymbol{\theta}_{t};
\mathcal{H}_{<t}
\right),
\qquad
\boldsymbol{\theta}_{t}
=
\boldsymbol{\theta}_{t-1}
+
\Delta\boldsymbol{\theta}_{t}^{\mathrm{c}},
\label{eq:constrained-update}
\end{equation}

where \(\Delta\boldsymbol{\theta}_{t}\) denotes the update on the
stage $t$, \(\mathcal{H}_{<t}\) represents historical information from previous stages, and \(\mathcal{C}\) constrains the current update according
to this historical information. In this way, the model limits interference
with previously learned knowledge while adapting to the current stage.
\\\noindent\textbf{Separating Task-Specific Adaptations:}
To avoid overwriting previous adaptations, methods such as CL-MoE~\citep{huai2025clmoe} and HiDe-LLaVA~\citep{guo2025hidellava} maintain specialized or expandable components for different task requirements. Their forward computation is represented as:
\begin{equation}
\mathbf{o}
=
f_{0}(\mathbf{x})
+
\sum_{m=1}^{M_t}
r_m(\mathbf{x})\mathcal{A}_m(\mathbf{x}),
\label{eq:separated-adaptation}
\end{equation}
where \(f_{0}\) denotes the frozen backbone,
\(\{\mathcal{A}_m\}_{m=1}^{M_t}\) denotes the adaptation components available
after stage \(t\), and \(r_m(\mathbf{x})\) controls the selection or
contribution of each component. In this way, task-specific adaptations are selectively activated to reduce interference across continual stages.
\\\noindent\textbf{Discussion:}
Equations~\ref{eq:constrained-update} and~\ref{eq:separated-adaptation}
represent two typical parameter-space strategies for mitigating catastrophic
forgetting in MCIT. Constraining shared parameter updates helps preserve
previously learned knowledge, but may restrict the flexibility required to
acquire new capabilities. Separating task-specific adaptations reduces direct
interference, but can duplicate reusable knowledge across different
components. More importantly, both strategies rely primarily on model
parameters and adaptation modules to organize continual learning, without
exploiting external knowledge that can provide domain information and
reasoning operations for individual instructions. Such knowledge can further
reveal capabilities that are reusable within and across tasks. Hence, an
effective MCIT framework should leverage external knowledge to guide capability
acquisition and reuse while preserving previously learned behaviors across
continual stages.

\section{Method}
\label{sec:method}

\method connects external knowledge with continual parameter adaptation through reusable capabilities. We first incrementally build a structured knowledge base, where retrieved domain knowledge provides generation guidance and reasoning knowledge provides ordered capability operations for composing capability-specific modules. Afterward, to enable stable reuse of recurring capabilities, \method introduces adaptive subspace recycling, which protects historically important directions while retaining residual capacity for subsequent adaptation.

\subsection{Knowledge Base Construction and Retrieval}
\label{sec:kb}

\noindent\textbf{Knowledge Card Construction:} At each continual stage \(j\), we construct a stage-local knowledge shard from the current training instructions in \(\mathcal{D}_j\). We first group instructions by their inferred answer schema, then encode them with a frozen embedding model \(\phi\) and cluster semantically similar instructions within each group. 
Representative instructions from each cluster, together with external search results that provide visual context and background information, are then summarized by an LLM.
This process produces three types of knowledge cards: domain cards contain relevant concepts, relations, and concise \textit{prompt guidance}; reasoning cards contain ordered capability operations; and format cards specify the expected answer schema. 
Detailed construction prompts and card schemas are provided in Appendix~\ref{app:knowledge}. Let \(\mathcal{K}^{j}_{c}\) denote the stage-\(j\) cards of type \(c\in\{\mathrm{dom},\mathrm{rea},\mathrm{fmt}\}\). At stage \(t\), the available knowledge is accumulated only from the observed stages:
\begin{equation}
    \mathcal{K}^{1:t}_{c}
    =
    \bigcup_{j=1}^{t}\mathcal{K}^{j}_{c},
    \qquad
    c\in\{\mathrm{dom},\mathrm{rea},\mathrm{fmt}\}.
    \label{eq:causal-kb}
\end{equation}
Accordingly, training and evaluation at stage \(t\) use only \(\mathcal{K}^{1:t}_{c}\), without access to future-stage instructions or cards. For subsequent retrieval, each domain and reasoning card is associated with one or more cluster prototypes obtained by averaging the embeddings of its source instructions. The same frozen embedding model \(\phi\) is used to encode retrieval queries.

\noindent\textbf{Knowledge Retrieval:} For each instruction \(\mathbf{q}\), \method first retrieves a domain card and a format card, whose information is then used as additional cues for reasoning-card retrieval.
For a candidate domain or reasoning card \(k\), the retrieval score combines similarity to its associated prototypes with type-specific card-content signals capturing semantic, lexical, and structured compatibility:
\begin{align}
    s_{\mathrm{proto}}(\mathbf{q},k)
    &=
    \max_{\mathbf{p}\in\mathcal{P}(k)}
    \cos\!\left(\phi(\mathbf{q}),\mathbf{p}\right), \\
    s_c(\mathbf{q},k)
    &=
    \alpha_c s_{\mathrm{proto}}(\mathbf{q},k)
    +(1-\alpha_c)
    \sum_{m\in\mathcal{M}_c}
    \lambda_{c,m}s_m(\mathbf{q},k),
    \qquad
    c\in\{\mathrm{dom},\mathrm{rea}\}.
    \label{eq:retrieval}
\end{align}
Here \(\mathcal{P}(k)\) contains the prototypes associated with card \(k\), while \(\mathcal{M}_c\) indexes the type-specific card-content signals used for card type \(c\). The coefficient \(\alpha_c\) balances prototype similarity and content matching, and \(\lambda_{c,m}\) weights each normalized signal \(s_m\). Domain retrieval uses card-text semantic similarity, keyword overlap, and concept overlap. Format cards are retrieved separately using keyword and answer-schema matching, with explicit response requirements prioritized. Reasoning retrieval further combines semantic, keyword, and concept matching with compatibility to the retrieved domain and format information. We retain the highest-scoring card of each type. All retrieval signals and fixed coefficients are detailed in Appendix~\ref{app:routing}. The left panel of Figure~\ref{fig:overview} illustrates this stage-wise knowledge construction and instruction-conditioned retrieval process.

\begin{figure*}[t]
    \centering
\vspace{-2mm}    
\includegraphics[width=\textwidth,pagebox=cropbox]{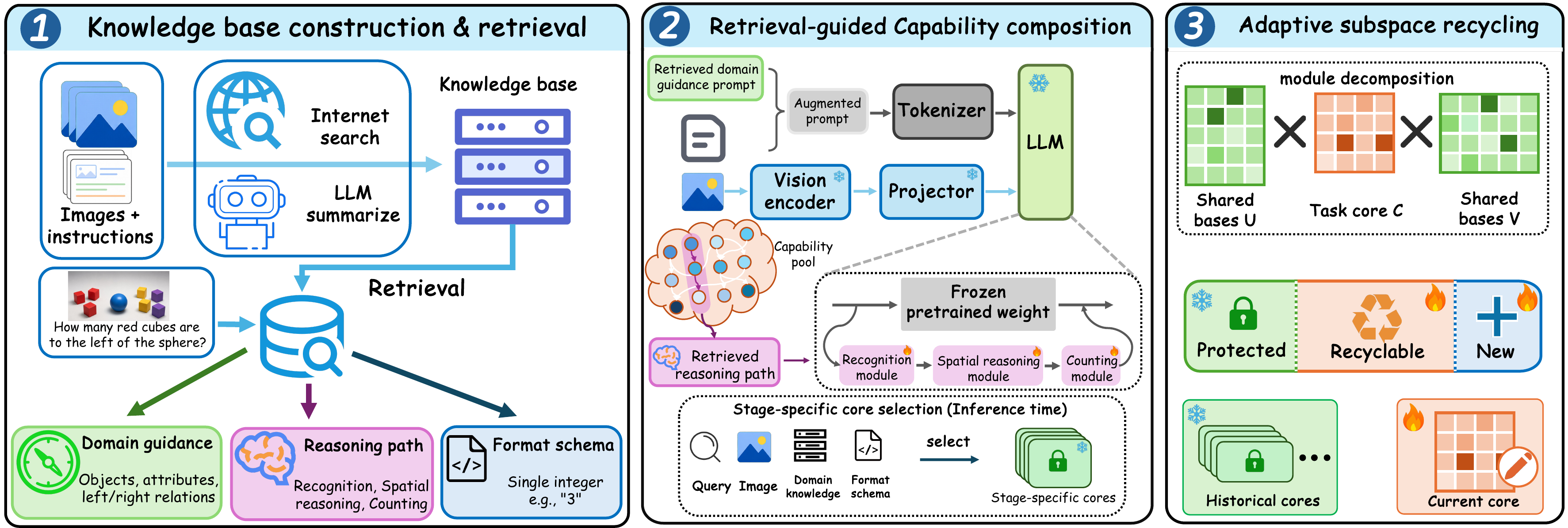}
\vspace{-6mm}
\caption{\small Overview of \method. (1) At each continual stage, a structured knowledge base is incrementally constructed from the current training instructions with external search and LLM summarization; the current instruction retrieves domain guidance, a reasoning path, and a format schema. (2) Domain guidance augments the model prompt, while the retrieved reasoning path selects and sequentially composes parameterized modules from a shared capability pool. At inference time, instruction, retrieved knowledge, visual, and path-transition cues select a stage-specific core for each activated capability. (3) Adaptive subspace recycling represents recurring capabilities with shared bases and stage-specific cores, protecting high-energy historical directions, recycling low-energy directions, and appending new directions for subsequent adaptation.}
    \label{fig:overview}
    \vspace{-6mm}
\end{figure*}

\subsection{Retrieval-Guided Capability Composition}
\label{sec:interface}

After knowledge retrieval, we keep the pretrained backbone frozen and maintain a shared capability set \(\mathcal{S}\). Each capability \(s\in\mathcal{S}\) represents a canonical reusable operation shared across continual stages, such as visual recognition, spatial reasoning, or counting, and is implemented by a capability-specific parameterized module in the adapted LLM layers. Retrieved knowledge determines which capabilities are activated for each instance and how they are composed.

\noindent\textbf{Domain Guidance:} Given the retrieved domain card \(k_{\mathrm{dom}}^*\), its \textit{prompt guidance} field is concatenated with the original instruction:
\begin{equation}
    \widetilde{\mathbf{q}}
    =
    \operatorname{concat}\!\left(
        \operatorname{guidance}(k_{\mathrm{dom}}^*),
        \mathbf{q}
    \right).
    \label{eq:guidance}
\end{equation}

\noindent\textbf{Capability Path:} The retrieved reasoning card specifies an ordered sequence of capability operations, each mapped to a canonical capability in \(\mathcal{S}\) during path construction. The resulting capabilities define an instance-specific capability path:
\begin{equation}
    \pi(\mathbf{q})=(s_1,s_2,\ldots,s_L),
    \qquad
    s_j\in\mathcal{S},
    \label{eq:path}
\end{equation}
where \(L\) is the path length and each \(s_j\) denotes a reusable capability. For example, an instruction involving object identification, spatial filtering, and counting may yield \(\text{visual recognition}\rightarrow\text{spatial reasoning}\rightarrow\text{counting}\).
During training at stage \(t\), capability \(s_j\) at adapted layer \(\ell\) is parameterized by \(A_{s_j,t}^{\ell}\).
% At adapted layer \(\ell\), the stage-\(t\) parameterization of capability \(s_j\) is denoted by \(A_{s_j,t}^{\ell}\).
Given an input representation \(\mathbf{x}\), the selected capability modules are applied sequentially along the retrieved path:
\begin{equation}
    \mathbf{z}_0=\mathbf{x},\qquad
    \mathbf{z}_j
    =
    \mathbf{z}_{j-1}
    +
    A_{s_j,t}^{\ell}\mathbf{z}_{j-1},
    \quad j=1,\ldots,L.
    \label{eq:sequential-composition}
\end{equation}
The adapted layer output is:
\begin{equation}
    \mathbf{o}
    =
    W_0^{\ell}\mathbf{x}
    +
    \mathbf{z}_L-\mathbf{z}_0,
    \label{eq:sequential-forward}
\end{equation}
where \(W_0^{\ell}\) is the frozen layer weight. Each capability operates on the representation produced by preceding capabilities, preserving the ordered computation specified by the retrieved reasoning path. The middle panel of Figure~\ref{fig:overview} illustrates this retrieval-guided capability composition process.

\subsection{Adaptive Subspace Recycling}
\label{sec:subspace-recycling}
\label{sec:factorization}
\label{sec:protection}

The same capability may recur across continual stages. Using a separate module for each recurrence prevents sharing, whereas updating a single shared module risks overwriting previously learned behaviors. To support stable reuse, \method represents each recurring capability with shared basis parameters and stage-specific cores, and applies adaptive subspace recycling to protect historically important directions while retaining trainable capacity for subsequent adaptation.

\noindent\textbf{Core Factorization:} For capability \(s\), we define its parameterization at layer \(\ell\) and stage \(t\) as \(A_{s,t}^{\ell}=U_{s}^{\ell}C_{s,t}^{\ell}V_{s}^{\ell}\), where \(A_{s,t}^{\ell}\in\mathbb{R}^{d_{\mathrm{out}}\times d_{\mathrm{in}}}\), \(U_s^{\ell}\in\mathbb{R}^{d_{\mathrm{out}}\times K_s}\), \(C_{s,t}^{\ell}\in\mathbb{R}^{K_s\times K_s}\), and \(V_s^{\ell}\in\mathbb{R}^{K_s\times d_{\mathrm{in}}}\). Here \(d_{\mathrm{in}}\) and \(d_{\mathrm{out}}\) denote the input and output dimensions, and \(K_s\) is the rank. \(U_s^{\ell}\) and \(V_s^{\ell}\) are shared across stages, while \(C_{s,t}^{\ell}\) is the stage-specific core of capability \(s\).

\noindent\textbf{Energy-Ordered Reparameterization:} Directly protecting individual basis coordinates is unreliable because equivalent low-rank factorizations can represent the same capability parameterization using different coordinates. For a recurring capability \(s\) at one adapted layer, we therefore first construct an orthonormal coordinate system. 
For clarity, we omit the layer superscript \(\ell\) and write \(U\) and \(V\) for the active shared bases.
Reduced QR decomposition gives:
\begin{equation}
    U=Q_U R_U,
    \qquad
    V^\top=Q_V R_V.
    \label{eq:qr}
\end{equation}
Let \(\mathcal{T}_s^{t-1}\) denote the previous stages in which capability \(s\) has a stored core. For each \(\tau\in\mathcal{T}_s^{t-1}\), the historical core is expressed in the QR coordinates as \(B_{s,\tau}=R_U C_{s,\tau}R_V^\top\). We then aggregate historical energy across these cores along the left and right coordinates:
\begin{equation}
    M_U
    =
    \sum_{\tau\in\mathcal{T}_s^{t-1}}
    B_{s,\tau}B_{s,\tau}^{\top}
    =
    P_U\Lambda_U P_U^\top, \qquad
    M_V
    =
    \sum_{\tau\in\mathcal{T}_s^{t-1}}
    B_{s,\tau}^{\top}B_{s,\tau}
    =
    P_V\Lambda_V P_V^\top.
    \label{eq:energy}
\end{equation}
Here \(P_U\) and \(P_V\) contain the corresponding eigenvectors, while \(\Lambda_U=\operatorname{diag}(\lambda_{U,1},\ldots)\) and \(\Lambda_V=\operatorname{diag}(\lambda_{V,1},\ldots)\) contain nonnegative eigenvalues sorted in descending order. Each eigenvalue measures the accumulated historical energy along the corresponding left or right direction. We rotate the shared bases and historical cores into these energy-ordered coordinates:
\begin{equation}
    \widetilde{U}=Q_U P_U,
    \qquad
    \widetilde{V}=P_V^\top Q_V^\top,
    \qquad
    \widetilde{C}_{s,\tau}=P_U^\top B_{s,\tau}P_V.
    \label{eq:rotation}
\end{equation}
Proposition~\ref{prop:exact-preservation} shows that the QR reparameterization and energy-based rotation leave every historical capability parameterization unchanged. These operations therefore only establish an energy-ordered coordinate system for determining which directions to protect and recycle.

\begin{figure*}[t]
    \centering
    \vspace{-3mm}    \includegraphics[width=\linewidth,pagebox=cropbox]{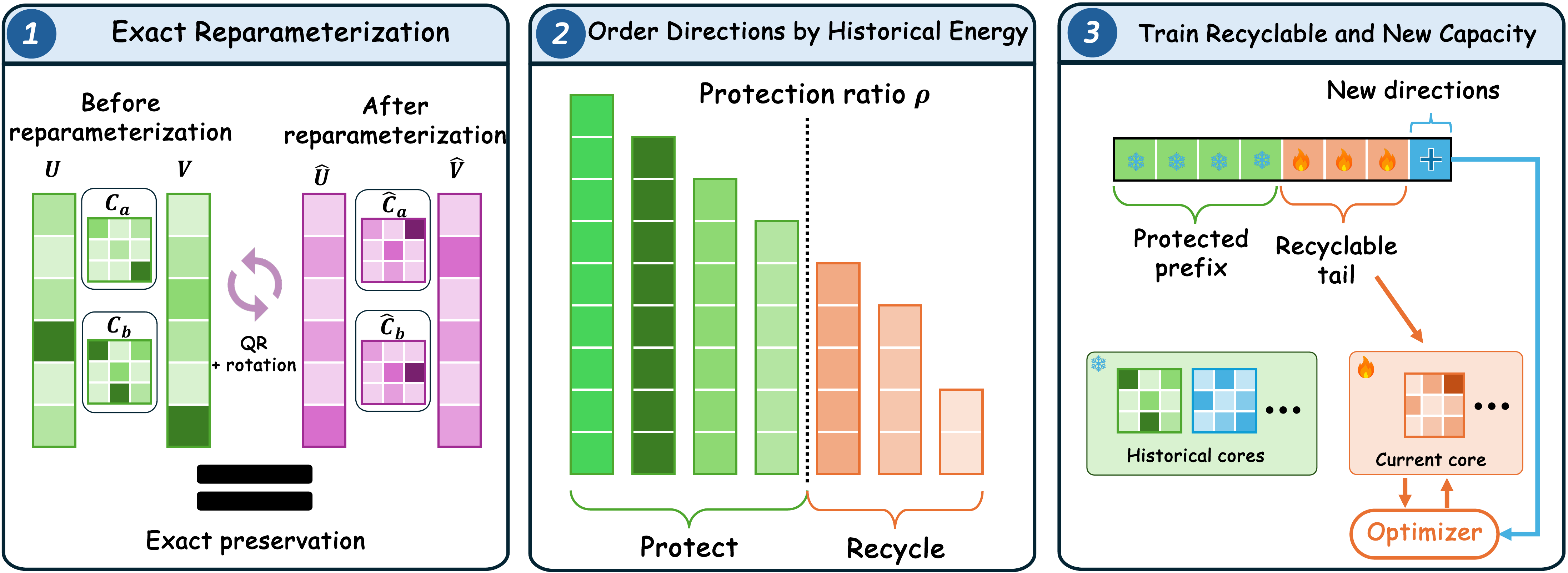}
    \vspace{-7mm}
    \caption{\small Adaptive subspace recycling. Historical capability parameterizations are re-expressed in energy-ordered coordinates without changing their represented functions. High-energy directions are protected, low-energy directions are recycled as trainable capacity, and new directions are appended for subsequent adaptation.}
    \label{fig:subspace-recycling}
    \vspace{-7mm}
\end{figure*}

\noindent\textbf{Protection and Recycling:} Let \(\rho\in(0,1]\) denote the fraction of historical energy to preserve. Since the eigenvalues are ordered by historical energy, the protected ranks \(p_U\) and \(p_V\) are chosen as the smallest values satisfying:
\begin{equation}
    p_U
    =
    \min\left\{
        p:
        \frac{\sum_{j=1}^{p}\lambda_{U,j}}
             {\sum_j\lambda_{U,j}}
        \geq\rho
    \right\},
    \qquad
    p_V
    =
    \min\left\{
        p:
        \frac{\sum_{j=1}^{p}\lambda_{V,j}}
             {\sum_j\lambda_{V,j}}
        \geq\rho
    \right\}.
    \label{eq:protected-rank}
\end{equation}
The prefixes \(\widetilde{U}[:,1{:}p_U]\) and \(\widetilde{V}[1{:}p_V,:]\) capture the dominant historical directions and are frozen during subsequent training. The remaining pre-existing directions form the recyclable subspace: they retain their current values but remain trainable for the new stage. Proposition~\ref{prop:residual-energy} shows that the historical energy outside each protected subspace is at most a fraction \(1-\rho\) of the corresponding total historical energy. A newly observed capability is initialized with active rank \(K_0\). When capability \(s\) recurs, let \(K^-\) denote its active rank before expansion and set \(K^+=\min(K^-+\delta,K_{\max})\). Historical cores and protected basis prefixes remain frozen, while the current-stage core and basis slices \(\widetilde{U}[:,p_U{+}1{:}K^+]\) and \(\widetilde{V}[p_V{+}1{:}K^+,:]\) are trainable. These trainable slices comprise the recyclable low-energy directions within the previous rank \(K^-\) and up to \(\delta\) newly appended directions. Historical cores are zero-padded after rank expansion, which preserves their represented parameterizations before further training as shown in Proposition~\ref{prop:exact-preservation}. Thus, recycling reuses low-energy directions as trainable capacity, while expansion provides additional directions when needed. The right panel of Figure~\ref{fig:overview} illustrates this protection, recycling, and expansion process.

\subsection{Summary}
\label{sec:summary}
\label{sec:core-routing}

During continual training, each capability in the retrieved path uses its current-stage cores \(C_{s,t}^{\ell}\). After stage \(t\), these cores are retained as a stage-specific realization of capability \(s\), together with routing signatures based on the instruction, retrieved domain and format information, visual input, and capability-transition context. At inference time, \method retrieves the knowledge cards and obtains the capability path \(\pi(\mathbf{q})=(s_1,\ldots,s_L)\). Let \(\mathcal{T}_s^t\) denote the stages up to \(t\) that contain a stored realization of capability \(s\). Each candidate \(\tau\in\mathcal{T}_s^t\) is scored as:
\begin{equation}
    r_{s,\tau}
    =
    \frac{
        \sum_{g\in\mathcal{G}(\mathbf{q},\mathbf{v})}
        \beta_g r_g(s,\tau)
    }{
        \sum_{g\in\mathcal{G}(\mathbf{q},\mathbf{v})}
        \beta_g
    },
    \label{eq:core-routing}
\end{equation}
where \(\mathcal{G}(\mathbf{q},\mathbf{v})\) contains the available routing cues, including instruction similarity, retrieved domain and format information, visual similarity, and capability-transition compatibility. Here \(r_g(s,\tau)\) is the normalized score for cue \(g\), and \(\beta_g\) is its fixed weight. We select \(\widehat{\tau}_s=\argmax_{\tau\in\mathcal{T}_s^t}r_{s,\tau}\) and use the corresponding per-layer cores \(C_{s,\widehat{\tau}_s}^{\ell}\) to recover the selected stage-specific realization of capability \(s\). The selected capability modules are then composed along \(\pi(\mathbf{q})\), enabling historical capability reuse.

\section{Experiments}
\label{sec:experiments}

\subsection{Experimental Setup}

\noindent\textbf{Datasets:}
We evaluate on the CoIN~\citep{chen2024coin} and UCIT~\citep{guo2025hidellava}
benchmarks. CoIN contains eight heterogeneous vision--language tasks:
ScienceQA~\citep{lu2022learn}, TextVQA~\citep{singh2019towards},
ImageNet~\citep{deng2009imagenet}, GQA~\citep{hudson2019gqa},
VizWiz~\citep{gurari2018vizwiz}, Grounding (RefCOCO)
~\citep{kazemzadeh2014referitgame,mao2016generation},
VQAv2~\citep{goyal2017making}, and OCR-VQA~\citep{mishra2019ocr}.
We follow the CoIN task order and evaluate after
each continual stage. UCIT contains six tasks in the order
ImageNet-R~\citep{hendrycks2021many},
ArxivQA~\citep{li2024multimodal}, VizWiz-Caption~\citep{gurari2018vizwiz},
IconQA~\citep{lu2021iconqa}, CLEVR~\citep{lindstrom2022clevr}, and
Flickr30k~\citep{plummer2015flickr30k}; we likewise evaluate after every stage.

\noindent\textbf{Metrics:}
Following~\cite{chen2024coin}, \(\mathcal{A}_{i,t}\) denotes the performance on task \(i\) after stage \(t\). Final average performance is defined as \(\overline{\mathcal{A}}=\frac1T\sum_{i=1}^{T}\mathcal{A}_{i,T}\). For \(t>1\), average forgetting is measured by \(B_t=\frac1{t-1}\sum_{i=1}^{t-1}(\mathcal{A}_{i,i}-\mathcal{A}_{i,t})\). Higher \(\overline{\mathcal{A}}\) and lower \(B_t\) are preferred, while negative \(B_t\) indicates backward transfer.

\noindent\textbf{Baselines:}
We compare with representative continual adaptation methods. LoRA-FT
~\citep{hu2022lora} provides a standard parameter-efficient baseline, while
O-LoRA~\citep{wang2023orthogonal} and MoELoRA~\citep{chen2024coin} study
orthogonalized updates and mixtures of low-rank experts. ModalPrompt
~\citep{zeng2025modalprompt}, CL-MoE~\citep{huai2025clmoe}, HiDe-LLaVA
~\citep{guo2025hidellava}, SEFE~\citep{chen2025sefe}, and SAME
~\citep{xie2026same} cover prompt-based adaptation, modular composition, and
mechanisms for preserving or routing previously learned knowledge.

\begin{table*}[t]
\vspace{-2mm}
\centering
\small
\renewcommand{\arraystretch}{1.08}
\setlength{\tabcolsep}{12.5pt}
\caption{\small Comparison with existing methods on UCIT. Task-wise results report final performance \(\mathcal{A}_{i,T}\). Best and second-best values among the main comparison methods are bolded and underlined, respectively; Vanilla RAG variants are shown for reference.}
\vspace{-3mm}
\label{tab:main}
\resizebox{\textwidth}{!}{%
\begin{NiceTabular}{@{}l|cccccc|c}
\toprule
\textbf{Method} & \textbf{ImageNet-R} & \textbf{ArxivQA} & \textbf{VizWiz}
& \textbf{IconQA} & \textbf{CLEVR} & \textbf{Flickr30k}
& \(\boldsymbol{\overline{\mathcal{A}}}\)~$\uparrow$ \\
\midrule
LoRA-FT~\citep{hu2022lora}
& 58.03 & 77.63 & 44.39 & 67.40 & 61.77 & \textbf{58.22} & 61.24 \\
O-LoRA~\citep{wang2023orthogonal}
& 77.50 & 78.07 & 44.50 & 63.13 & \underline{64.73} & \underline{58.16} & 64.35 \\
MoELoRA~\citep{chen2024coin}
& 70.07 & 77.70 & 44.69 & 50.03 & 54.03 & 57.34 & 58.98 \\
ModalPrompt~\citep{zeng2025modalprompt}
& 51.07 & 87.27 & 48.11 & 39.23 & 46.57 & 42.93 & 52.53 \\
CL-MoE~\citep{huai2025clmoe}
& 66.33 & 77.00 & 44.78 & 51.87 & 53.53 & 57.42 & 58.49 \\
HiDe-LLaVA~\citep{guo2025hidellava}
& \underline{84.03} & 90.73 & 44.43 & 58.93 & 41.37 & 54.25 & 62.29 \\
SEFE~\citep{chen2025sefe}
& 80.83 & 78.00 & 47.01 & \textbf{69.63} & \textbf{65.83} & 57.92 & 66.54 \\
SAME~\citep{xie2026same}
& 83.83 & \underline{91.40} & \underline{51.33} & 65.27 & 53.50 & 57.43 & \underline{67.12} \\
\midrule
\rowcolor{gray!8}
{LoRA-FT + Vanilla RAG}
& 37.33 & 79.06 & 44.57 & 71.13 & 63.07 & 58.07 & 58.87 \\
\rowcolor{gray!8}
{SAME + Vanilla RAG}
& 82.93 & 87.60 & 52.61 & 60.64 & 53.33 & 57.32 & 65.74 \\
\midrule
\rowcolor{blue!6}
\method (Ours) & \textbf{85.57} & \textbf{92.60} & \textbf{60.56}
& \underline{68.53} & 60.57 & 56.09 & \textbf{70.65} \\
\bottomrule
\end{NiceTabular}
}\vspace{-6mm}
\end{table*}

\noindent\textbf{Implementation details:} We use LLaVA-v1.5-7B with the CLIP ViT-L/14-336 visual encoder~\citep{liu2023visual,radford2021learning}. We adapt the attention projections while keeping the backbone frozen. Each stage is trained for one epoch with a batch size of 8 and a learning rate of \(2\times10^{-4}\). The factorized capability modules are initialized with rank \(K_0=8\) and expanded by \(\delta=1\) whenever a capability recurs, with \(\rho=0.99\); the maximum active rank is 13 on UCIT and 15 on CoIN. Qwen3-Embedding-4B encodes instructions and structured knowledge cards. For each instruction, retrieval selects one domain, reasoning, and format card. The knowledge base is constructed offline using external search and an LLM, and at stage \(t\), only shards from observed stages \(1{:}t\) are available during training and evaluation. Retrieval and core-routing coefficients are provided in Appendix~\ref{app:routing}.

\subsection{Benchmark Comparison}

\begin{table*}[t]
\vspace{-2mm}
\centering
\small
\renewcommand{\arraystretch}{1.08}
\setlength{\tabcolsep}{5.2pt}
\caption{\small Comparison with existing methods on CoIN. Task-wise results report final performance \(\mathcal{A}_{i,T}\). Best and second-best values among the main comparison methods are bolded and underlined, respectively; Vanilla RAG variants are shown for reference.}
\vspace{-3mm}
\label{tab:coin}
\resizebox{\textwidth}{!}{%
\begin{NiceTabular}{@{}l|cccccccc|c}
\toprule
\textbf{Method} & \textbf{ScienceQA} & \textbf{TextVQA} & \textbf{ImageNet}
& \textbf{GQA} & \textbf{VizWiz} & \textbf{Grounding} & \textbf{VQAv2}
& \textbf{OCR-VQA} & \(\boldsymbol{\overline{\mathcal{A}}}\)~$\uparrow$ \\
\midrule
LoRA-FT~\citep{hu2022lora}
& 26.00 & 25.38 & 28.51 & 33.07 & 26.52 & 0.10 & 40.00 & 52.92 & 29.06 \\
O-LoRA~\citep{wang2023orthogonal}
& 75.40 & 52.89 & 71.85 & 47.30 & 37.35 & 7.10 & 61.85 & 61.20 & 51.87 \\
MoELoRA~\citep{chen2024coin}
& 62.02 & 52.05 & 37.21 & 53.12 & 43.32 & 33.22 & 57.92 & \underline{65.75} & 50.58 \\
ModalPrompt~\citep{zeng2025modalprompt}
& 68.42 & 56.40 & 41.13 & 61.11 & 50.13 & 36.69 & 66.90 & 59.68 & 55.06 \\
CL-MoE~\citep{huai2025clmoe}
& 73.28 & 59.94 & 31.80 & 60.22 & 46.98 & \underline{64.48} & \textbf{67.36} & 62.39 & 58.31 \\
SEFE~\citep{chen2025sefe}
& 75.35 & 58.66 & 83.10 & 54.25 & 48.85 & 16.75 & 65.35 & \textbf{66.25} & 58.57 \\
HiDe-LLaVA~\citep{guo2025hidellava}
& 73.20 & 56.92 & 69.28 & \underline{61.33} & 50.76 & 59.18 & \underline{67.12} & 64.76 & 62.82 \\
SAME~\citep{xie2026same}
& \underline{78.35} & \underline{60.69} & \underline{90.21} & \textbf{61.70}
& \underline{54.13} & {59.87} & 66.04 & 63.59 & \underline{66.82} \\
\midrule
\rowcolor{gray!8}
{LoRA-FT + Vanilla RAG}
& 24.16& 26.14& 18.69& 39.60& 24.85& 6.74& 48.09& 50.88& 29.89  \\
\rowcolor{gray!8}
{SAME + Vanilla RAG}
& 66.51& 60.26& 90.25& 61.55& 61.23& 59.79& 65.98& 60.70& 65.78 \\
\midrule
\rowcolor{blue!6}
\method (Ours)
& \textbf{79.70} & \textbf{61.89} & \textbf{96.89} & 61.12
& \textbf{58.12} & \textbf{71.18} & 65.97 & 57.90 & \textbf{69.10} \\
\bottomrule
\end{NiceTabular}%
}\vspace{-7mm}
\end{table*}
Tables~\ref{tab:main} and~\ref{tab:coin} report final task-wise performance on UCIT and CoIN, respectively. On UCIT, \method achieves the highest final average performance of 70.65, outperforming SAME by 3.53 points, with the best results on ImageNet-R, ArxivQA, and VizWiz. Although several baselines remain stronger on individual tasks such as IconQA, CLEVR, and Flickr30k, \method delivers the strongest overall performance across the heterogeneous task stream. On the longer CoIN stream, \method achieves an average of 69.10, exceeding SAME by 2.28 points, and obtains the best results on ScienceQA, TextVQA, ImageNet, VizWiz, and visual grounding. Other methods remain competitive on GQA, VQAv2, and OCR-VQA. Overall, the consistent gains across both benchmarks demonstrate the effectiveness of \method across diverse continual tasks.

Additionally, we investigate whether the gains of \method can be attributed merely to access to external knowledge. To isolate this factor, we augment LoRA-FT and SAME with a vanilla RAG pipeline~\citep{lewis2020retrieval}, giving them access to the same knowledge base as \method during both training and evaluation. Despite this matched knowledge access, \method continues to substantially outperform both RAG-augmented baselines. This result suggests that the improvement does not stem simply from retrieving additional external information. Rather, the advantage lies in how \method operationalizes the retrieved knowledge: instead of treating retrieval as auxiliary context, \method uses it to identify, compose, and reuse instance-relevant capabilities across continual stages. In this way, retrieved knowledge serves not only as additional evidence for prediction, but also as a structured interface for guiding capability selection and transfer throughout continual learning.

\begin{figure}[ht] \centering \vspace{-3mm}\includegraphics[width=\textwidth,pagebox=cropbox]{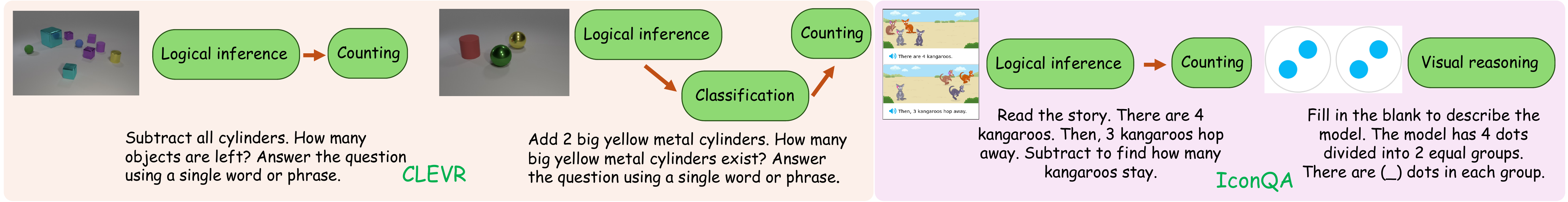} \vspace{-7mm} 
\caption{\small Instance-level capability paths for examples sampled from the UCIT test set. Samples within the same task may require different capability compositions, while shared capabilities can be reused across tasks.}
\label{fig:case-paths} \vspace{-7mm} \end{figure}

\subsection{Further Analysis}
\noindent\textbf{Instance-Level Capability Paths:} Figure~\ref{fig:case-paths} illustrates capability paths derived from retrieved reasoning knowledge. The two CLEVR examples show that different instructions within the same task can require different capability compositions: one follows logical inference \(\rightarrow\) counting, while the other additionally requires classification before counting. The IconQA example shares the logical inference \(\rightarrow\) counting path with CLEVR, showing that capability compositions can also be reused across different tasks. These examples illustrate how reasoning knowledge provides instance-level structure for capability composition and reuse in \method.
\\\noindent\textbf{Continual Retention:}
Figure~\ref{fig:forgetting-a} shows that \method maintains consistently low
forgetting throughout the UCIT stream. It exhibits slight backward transfer
after stages 2 and 3, while average forgetting remains only 0.14, 0.88, and
0.70 after stages 4, 5, and 6, respectively. Compared with the substantially
larger forgetting accumulated by most baselines, these results indicate that
previously learned capabilities remain stable as new stages are introduced.
\begin{figure*}[t]
    \vspace{-3.5mm}
    \centering
    \begin{subfigure}[t]{0.485\textwidth}
        \centering
        \includegraphics[width=\linewidth]{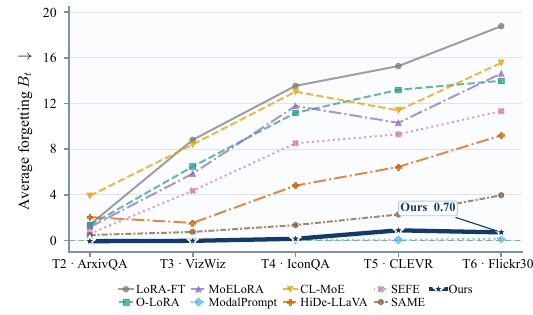}
        \vspace{-5mm}
        \caption{Average forgetting across stages.}
        \vspace{-3mm}
        \label{fig:forgetting-a}
    \end{subfigure}
    \hfill
    \begin{subfigure}[t]{0.485\textwidth}
        \centering
        \includegraphics[width=\linewidth]{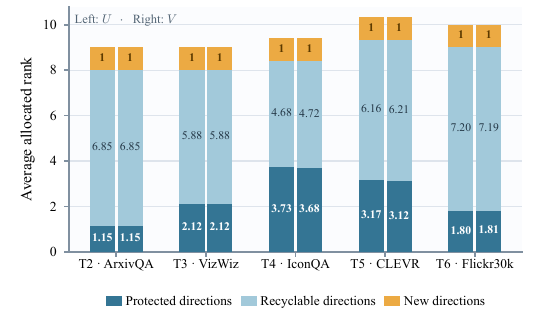}
        \vspace{-5mm}
        \caption{Subspace allocation under adaptive protection.}
        \vspace{-3mm}
        \label{fig:forgetting-b}
    \end{subfigure}
    \caption{\small Continual retention and adaptive subspace allocation on UCIT.
    Panel (a) reports average forgetting after each stage, where lower values
    indicate stronger retention. Panel (b) visualizes the protected,
    recyclable, and newly appended basis directions averaged over recurring
    capability-projection pairs.}
    \label{fig:forgetting}
    \vspace{-6.5mm}
\end{figure*}
\\\noindent\textbf{Subspace Allocation:}
Figure~\ref{fig:forgetting-b} summarizes the realized subspace allocation
across the five updates after the initial stage, averaged over recurring
capability-projection pairs. Since the protection rule in
\Eqref{eq:protected-rank} is applied independently to each capability and
adapted projection, the protected rank varies across updates rather than using
a fixed coordinate count. Despite the high protection ratio \(\rho=0.99\),
the protected prefixes occupy only about 14\%--44\% of the pre-expansion rank,
leaving roughly 56\%--86\% of the existing directions recyclable before new
capacity is added. This shows that historical energy is concentrated in
a compact subspace, allowing \method to preserve dominant historical
directions while retaining substantial capacity for subsequent adaptation.
\\\noindent\textbf{Ablation Study:} Table~\ref{tab:component-ablation} evaluates each component of \method. Removing domain guidance and eliminating reasoning-guided ordering by replacing the sequential path with parallel capability composition both reduce performance, confirming the complementary benefits of generation guidance and ordered capability execution. Disabling adaptive subspace recycling leads to substantial degradation, demonstrating the importance of preserving historical capability information while retaining sufficient capacity for subsequent adaptation.
Removing stage-specific core routing and forcing activated capabilities to use final-stage cores further degrades performance, as cores capture distinct stage-specific realizations, highlighting the need for effective routing in task-free capability reuse.
\\\noindent\textbf{Protection Ratio Sensitivity:} Table~\ref{tab:rho-ablation} evaluates the sensitivity of \method to the historical-energy protection ratio \(\rho\) on UCIT using a held-out validation set, with the realized protected rank indicating how much subspace is preserved under each setting. Lower protection ratios (\(\rho\leq0.90\)) retain only
1.01--1.34 directions on average and yield lower performance than the default
setting. Increasing \(\rho\) to 0.99 raises the protected rank to 2.53 and
achieves the best average performance of 70.67, indicating that most historical
energy is concentrated within a compact subspace. In contrast, full protection
at \(\rho=1.0\) increases the protected rank sharply to 8.37 and reduces
performance to 66.52, as substantially less residual capacity remains for new
adaptation. These results favor a high but non-complete protection ratio that
balances historical retention and plasticity.

\begin{table*}[ht]
\vspace{-3mm}
\begin{minipage}[t]{0.4\textwidth}
\centering
\small
\renewcommand{\arraystretch}{1.08}
\setlength{\tabcolsep}{4pt}
\caption{\small Component ablations of \method on UCIT. Higher is better.}
\vspace{-3mm}
\label{tab:component-ablation}
\begin{tabular}{lc}
\toprule
\textbf{Variant} & \(\boldsymbol{\overline{\mathcal{A}}}\)~$\uparrow$ \\
\midrule
\rowcolor{blue!6}
\method & \textbf{70.65} \\
w/o Domain Guidance & 70.02 \\
w/o Reasoning-Guided Ordering & 70.08 \\
w/o Adaptive Subspace Recycling & 49.88 \\
w/o Stage-specific Core Routing & 31.50 \\
\bottomrule
\end{tabular}
\end{minipage}
\hfill
\begin{minipage}[t]{0.58\textwidth}
\centering
\small
\renewcommand{\arraystretch}{1.08}
\setlength{\tabcolsep}{16pt}
\caption{\small Sensitivity to the historical-energy ratio \(\rho\), with the resulting protected rank.}
\vspace{-3mm}
\label{tab:rho-ablation}
\begin{tabular}{ccc}
\toprule
\(\boldsymbol{\rho}\)
& \textbf{Protected Rank}
& \(\boldsymbol{\overline{\mathcal{A}}}\)~$\uparrow$ \\
\midrule
0.50 & 1.01 & 67.49 \\
0.70 & 1.10 & 66.91 \\
0.90 & 1.34 & 69.33 \\
\rowcolor{blue!6}
0.99 (default) & 2.53 & \textbf{70.67} \\
1.00 & 8.37 & 66.52 \\
\bottomrule
\end{tabular}
\end{minipage}\vspace{-7mm}
\end{table*}

\section{Conclusion}
We presented \method, a retrieval-guided framework for multimodal continual instruction tuning that leverages external knowledge to organize capability composition and reuse across continual stages. Retrieved domain knowledge provides generation guidance, while reasoning knowledge specifies ordered capability operations for instance-level composition, with format knowledge further supporting retrieval and routing. To enable stable cross-stage reuse, adaptive subspace recycling represents recurring capabilities through shared bases and stage-specific cores, protecting high-energy historical directions while retaining residual capacity for subsequent adaptation. Experiments on UCIT and CoIN show that \method consistently outperforms representative baselines while maintaining strong continual retention and effective adaptation to newly observed tasks.\looseness -1
\\\noindent\textbf{Limitations.} The current knowledge base represents external information as concise textual cards and does not model richer sources such as tables, diagrams, and structured knowledge graphs. Extending knowledge construction and retrieval to these forms is a natural direction for future work.

\subsection*{AI Use Statement}
In this work, we used generative AI tools for research execution. Specifically, GPT-4o mini was used within \method for knowledge-base construction, including knowledge-card generation, capability canonicalization, domain-card consolidation, and prompt-guidance generation. Additionally, we used generative AI tools to check language, notation consistency, and presentation clarity throughout the manuscript. We reviewed and verified all AI-assisted work. We take responsibility for the final content of this work, including text, claims, and artifacts produced with the aid of generative AI.

\bibliography{iclr2027_conference}
\bibliographystyle{iclr2027_conference}

\clearpage
\appendix
\section*{\centering Appendix}

\section{Training and Inference Algorithms}
\label{app:algorithms}

Algorithms~\ref{alg:training} and~\ref{alg:inference} summarize the complete procedure for stage-wise training and task-free inference.

\begin{algorithm}[ht]
\caption{Stage-wise training of \method}
\label{alg:training}
\begin{algorithmic}[1]
\REQUIRE Stage data \(\mathcal{D}_t\), knowledge shard \(\mathcal{K}^{t}\), and parameters retained from stages \(1{:}t-1\)
\STATE Update the available knowledge:
\(\mathcal{K}^{1:t}\leftarrow\mathcal{K}^{1:t-1}\cup\mathcal{K}^{t}\)
\FORALL{\((\mathbf v,\mathbf q,\mathbf y)\in\mathcal{D}_t\)}
    \STATE Retrieve one domain, format, and reasoning card
    \STATE Form \(\widetilde{\mathbf q}\) and obtain the capability path
    \(\pi(\mathbf q)=(s_1,\ldots,s_L)\)
\ENDFOR
\STATE Collect the capabilities activated by the current-stage paths
\FORALL{capabilities \(s\) activated at stage \(t\)}
    \IF{\(s\) has a stored historical realization}
        \STATE Apply the QR reparameterization and energy-based rotation in
        \eqref{eq:qr}--\eqref{eq:rotation}
        \STATE Compute \(p_U,p_V\) and expand the active rank by up to \(\delta\)
    \ELSE
        \STATE Initialize the shared bases with active rank \(K_0\)
    \ENDIF
    \STATE Initialize the current-stage cores \(C_{s,t}^{\ell}\)
\ENDFOR
\STATE Optimize the current-stage cores and trainable basis directions using
\eqref{eq:objective} and \eqref{eq:trainable-slices}
\STATE Keep historical cores and protected basis directions fixed
\STATE Store the current capability realizations and their routing signatures
\end{algorithmic}
\end{algorithm}

\begin{algorithm}[ht]
\caption{Task-free inference with \method}
\label{alg:inference}
\begin{algorithmic}[1]
\REQUIRE Query--image pair \((\mathbf q,\mathbf v)\), available knowledge
\(\mathcal{K}^{1:t}\), and stored capability realizations
\STATE Retrieve the domain, format, and reasoning cards
\STATE Form \(\widetilde{\mathbf q}\) and obtain the capability path
\(\pi(\mathbf q)=(s_1,\ldots,s_L)\)
\FOR{\(j=1,\ldots,L\)}
    \STATE Score the stored realizations of \(s_j\) using
    \eqref{eq:core-routing}
    \STATE Select \(\widehat{\tau}_{s_j}\) and use the corresponding per-layer
    cores \(C_{s_j,\widehat{\tau}_{s_j}}^{\ell}\)
\ENDFOR
\STATE Compose the selected capability modules in path order and decode the response
\end{algorithmic}
\end{algorithm}

\section{Knowledge Base Construction}
\label{app:knowledge}

\noindent\textbf{Overview:}
For each continual stage \(j\), we construct a knowledge shard from the training instructions in \(\mathcal{D}_j\). Instructions are first grouped by their inferred answer schemas and then clustered by semantic similarity within each group using the embedding model \(\phi\). For each cluster, representative instructions and externally retrieved evidence, which may provide visual context and background information, are summarized by an LLM into domain, reasoning, and format cards. Domain cards contain reusable knowledge and prompt guidance, reasoning cards specify ordered capability operations and their applicable contexts, and format cards define answer schemas and output constraints. Generated capabilities are canonicalized against the existing capability registry to promote reuse, while semantically compatible domain cards are consolidated to reduce redundancy. Algorithm~\ref{alg:kb-construction} summarizes the complete procedure.

\begin{algorithm}[ht]
\small
\caption{Construction of the knowledge base.}
\label{alg:kb-construction}
\begin{algorithmic}[1]
\REQUIRE Training sets \(\{\mathcal{D}_j\}_{j=1}^{T}\), embedding model \(\phi\), external search, LLM \(\mathcal{L}\), capability registry \(\mathcal{R}\), allowed answer schemas \(\mathcal{Y}\)
\ENSURE Knowledge shards
\(\{\mathcal{K}_{\mathrm{dom}}^j,
\mathcal{K}_{\mathrm{rea}}^j,
\mathcal{K}_{\mathrm{fmt}}^j\}_{j=1}^{T}\)
\FOR{\(j=1,\ldots,T\)}
    \STATE Collect instructions \(\mathcal{Q}_j\) from \(\mathcal{D}_j\)
    \STATE Infer answer schemas and group instructions accordingly
    \STATE Encode instructions with \(\phi\) and cluster semantically similar instructions within each schema group
    \FOR{each instruction cluster \(\mathcal{C}_a\)}
        \STATE Summarize the cluster intent and relevant concepts
         \STATE Retrieve external evidence \(\mathcal{E}_a\) using queries constructed from \\the task name, cluster intent, and concepts
        \STATE Generate domain, reasoning, and format cards using Prompt~A
        \STATE Canonicalize generated capabilities against \(\mathcal{R}\) using Prompt~B
        \STATE Update \(\mathcal{R}\) with retained new capabilities
        \STATE Validate grounding, schema consistency, and path completeness
    \ENDFOR
    \STATE Identify semantically compatible domain-card groups
    \STATE Consolidate compatible domain cards using Prompt~C
    \STATE Generate domain-card prompt guidance using Prompt~D
    \STATE Encode domain and reasoning card representations with \(\phi\)
    \FOR{each instruction cluster \(\mathcal{C}_a\)}
        \STATE
        \(\displaystyle
        \mathbf{p}_a \leftarrow
        \operatorname{norm}\!\left(
        \frac{1}{|\mathcal{C}_a|}
        \sum_{\mathbf{q}_i\in\mathcal{C}_a}
        \phi(\mathbf{q}_i)
        \right)\)
        \STATE Associate \(\mathbf{p}_a\) with its resulting domain and reasoning cards
    \ENDFOR
    \STATE Consolidated domain cards retain the prototypes of all member clusters
    \STATE Store the resulting cards in
    \(\mathcal{K}_{\mathrm{dom}}^j\),
    \(\mathcal{K}_{\mathrm{rea}}^j\), and
    \(\mathcal{K}_{\mathrm{fmt}}^j\)
\ENDFOR
\STATE At stage \(t\), expose only
\(\mathcal{K}_{c}^{1:t}=\bigcup_{j=1}^{t}\mathcal{K}_{c}^{j}\),
\(c\in\{\mathrm{dom},\mathrm{rea},\mathrm{fmt}\}\)
\end{algorithmic}
\end{algorithm}

\subsection{Knowledge-Card Generation}

\noindent\textbf{External Search:} For each instruction cluster, search queries are constructed from the task name, cluster intent, and concepts. The retrieved textual evidence provides relevant background information and may include descriptions of visual context associated with the task and concepts.

\noindent\textbf{Card Generation:} The instruction cluster, retrieved evidence, currently available capabilities, and answer-schema constraints are provided to GPT-4o mini~\citep{openai2024gpt4omini}. Prompt~A generates structured domain, reasoning, and format cards, while Prompts~B-D support capability canonicalization and subsequent card refinement when applicable.

\noindent\textbf{System prompt.}
\begin{lstlisting}
You construct source-grounded knowledge cards for multimodal continual learning.
Generate reusable knowledge and a complete reasoning path grounded in the provided evidence and instruction cluster.
Do not answer any sample instruction.
Return compact valid JSON only.
Prefer an existing allowed capability whenever it performs the same reusable operation.
Introduce a new capability only when the required operation is not covered by the existing capability list.
Capabilities must be atomic, reusable operations rather than task-specific procedures.
Decompose composite reasoning processes into separate capability steps.
Every answer schema must exactly match the allowed schema list.
\end{lstlisting}

\noindent\textbf{User prompt template.}
\begin{lstlisting}
Task: <task_name>
Reasoning intent: <intent>
Schema hint: <schema_hint>
Concept hints: <concept_hints>

Instruction cluster:
- <instruction_1>
- ...
- <instruction_n>

Allowed capabilities:
<capability_list>

Allowed schemas:
<schema_list>

Retrieved evidence:
<source_records>

Return JSON with the following structure:
{
  "domain_cards": [
    {
      "name": "",
      "knowledge": "",
      "concepts": [],
      "keywords": [],
      "source_refs": []
    }
  ],
  "reasoning_cards": [
    {
      "name": "",
      "domains": [],
      "schemas": [],
      "pattern": [
        {
          "capability": "",
          "operation": ""
        }
      ],
      "source_refs": []
    }
  ],
  "format_cards": [
    {
      "schema": "",
      "format_template": "",
      "validation_rule": ""
    }
  ]
}

Requirements:
- Generalize to reusable knowledge rather than instruction-specific facts.
- Do not include final answers to any instruction.
- Domain and reasoning cards must be supported by the retrieved evidence.
- The reasoning path must contain only atomic capabilities.
- The reasoning path must be sufficient to complete the requested reasoning intent.
- Reuse an existing capability whenever its core operation already covers a required step.
- At least one reasoning card and one format card must match the requested schema.
- Include only operations supported by the instruction cluster and retrieved evidence.
- Use the shortest complete reasoning path.
\end{lstlisting}

\subsection{Capability Canonicalization}

Reasoning cards from different instruction clusters may describe the same reusable operation using different names. We therefore compare newly generated capabilities with the existing registry and reuse a canonical capability only when their core operations and output roles are equivalent.

\paragraph{Prompt B: capability canonicalization.}

\noindent\textbf{System prompt.}
\begin{lstlisting}
Determine whether each generated capability represents the same reusable operation as an existing capability.
Use REUSE only when the core operation and output role are substantially equivalent.
Related, broader, narrower, or task-specific operations must remain separate.
Return valid JSON only.
For each candidate, return its capability, decision, canonical_capability, and confidence.
The decision must be either REUSE or KEEP_SEPARATE.
\end{lstlisting}

\noindent\textbf{Input/output template.}
\begin{lstlisting}
Input:
{
  "candidates": [
    {
      "capability": "<generated_capability>",
      "operation": "<operation_description>",
      "schemas": <schemas>
    }
  ],
  "existing_capabilities": [
    {
      "capability": "<canonical_capability>",
      "operation": "<operation_description>"
    }
  ]
}

Output:
{
  "decisions": [
    {
      "capability": "",
      "decision": "",
      "canonical_capability": "",
      "confidence": 0.0
    }
  ]
}
\end{lstlisting}

Only high-confidence \texttt{REUSE} decisions whose canonical capability
already exists in the registry are applied. This canonicalization provides
persistent capability identities across knowledge cards constructed at
different continual stages.

\subsection{Domain-Card Consolidation}

Domain cards from different instruction clusters may encode equivalent reusable knowledge. We first form conservative candidate groups using semantic and embedding compatibility, and then use the LLM to determine whether cards within each group express the same knowledge rule. Cards with distinct rules remain separate even when they share similar terminology or answer formats.

\paragraph{Prompt C: domain-card consolidation.}

\noindent\textbf{System prompt.}
\begin{lstlisting}
Merge Domain knowledge cards only when they express the same semantic intent and the same reusable knowledge rule.
A shared answer schema, reasoning path, name, or superficial vocabulary is not sufficient for merging.
Use representative instructions as the primary evidence when the knowledge description is ambiguous.
Do not invent new knowledge.
Every input card must appear in exactly one output group.
Return valid JSON only.
For each group, return member_uids, name, knowledge, concepts, and keywords.
\end{lstlisting}

\noindent\textbf{Input/output template.}
\begin{lstlisting}
Input:
{
  "cards": [
    {
      "id": "",
      "knowledge": "",
      "concepts": [],
      "semantic_intent": "",
      "representative_instructions": []
    }
  ]
}

Output:
{
  "groups": [
    {
      "member_ids": [],
      "name": "",
      "knowledge": "",
      "concepts": [],
      "keywords": []
    }
  ]
}
\end{lstlisting}

This consolidation reduces redundant domain knowledge while preserving
distinct semantic rules that may share similar terminology or response
formats.

\subsection{Prompt-Guidance Generation}

Each consolidated domain card is converted into concise operational guidance that can be prepended to the retrieved instruction through \eqref{eq:guidance}. Prompt~D restricts the guidance to card-grounded concepts and rules, without introducing sample-specific answers or output-format instructions.

\paragraph{Prompt D: domain-card guidance.}

\noindent\textbf{System prompt.}
\begin{lstlisting}
Create prompt_guidance for every supplied Domain knowledge card.
Ground the guidance only in the supplied knowledge and concepts.
Rewrite the card as one concise operational instruction describing the evidence, concepts, distinctions, or stable rules the model should use.
Do not reveal an answer, rationale, option, or instruction-specific fact.
Do not include answer-format instructions.
Avoid generic advice and use concrete operations only when supported by the card.
Return one concise English paragraph for each card in valid JSON.
\end{lstlisting}

\noindent\textbf{Input/output template.}
\begin{lstlisting}
Input:
{
  "cards": [
    {
      "id": "",
      "knowledge": "",
      "concepts": []
    }
  ]
}

Output:
{
  "items": [
    {
      "id": "",
      "prompt_guidance": ""
    }
  ]
}
\end{lstlisting}

\subsection{Retrieval Representations}

The final knowledge store maintains card-text representations together with instruction-cluster prototypes. Domain-card representations contain names, knowledge, concepts, and keywords, while reasoning-card representations contain applicable domains, schemas, and ordered capability operations. Both are encoded by the frozen embedding model \(\phi\).

For each instruction cluster \(\mathcal{C}_a\), its prototype is obtained by averaging the instruction embeddings followed by \(\ell_2\)-normalization:
\begin{equation}
    \mathbf{p}_a
    =
    \operatorname{norm}\!\left(
    \frac{1}{|\mathcal{C}_a|}
    \sum_{\mathbf{q}_i\in\mathcal{C}_a}
    \phi(\mathbf{q}_i)
    \right),
    \label{eq:cluster-prototype-app}
\end{equation}
where \(\operatorname{norm}(\cdot)\) denotes \(\ell_2\)-normalization. Each domain or reasoning card retains the prototypes of its source instruction clusters; a consolidated domain card may therefore contain multiple prototypes. We denote this set by \(\mathcal{P}(k)\). Retrieval combines prototype similarity with the card-content signals described in Section~\ref{sec:kb} and Appendix~\ref{app:routing}.

\section{Fixed Retrieval and Core-Selection Coefficients}
\label{app:routing}

Table~\ref{tab:routing-coefficients} summarizes the fixed coefficients used by \method across all continual tasks and stages. Domain and reasoning retrieval combine instruction-cluster prototype similarity with card-content matching using \(\alpha_{\mathrm{dom}}=\alpha_{\mathrm{rea}}=0.80\) in \eqref{eq:retrieval}, while format retrieval is performed separately using lexical and answer-schema matching. The same coefficients are used throughout the continual stream.

\noindent\textbf{Retrieval Signals:}
At stage \(t\), retrieval accesses only the knowledge accumulated from observed stages \(1{:}t\). For a domain or reasoning card \(k\), the prototype score is defined as:
\begin{equation}
    s_{\mathrm{proto}}(\mathbf q,k)
    =
    \max_{\mathbf p\in\mathcal P(k)}
    \cos\!\left(
        \phi(\mathbf q),\mathbf p
    \right),
    \label{eq:prototype-score-app}
\end{equation}
where \(\mathcal{P}(k)\) contains the instruction-cluster prototypes associated with card \(k\). In addition to prototype similarity, each card type uses a weighted combination of normalized content-matching signals:
\begin{equation}
    h_c(\mathbf q,k)
    =
    \sum_{m\in\mathcal{M}_c}
    \lambda_{c,m}s_m(\mathbf q,k),
    \qquad
    c\in\{\mathrm{dom},\mathrm{rea},\mathrm{fmt}\},
    \label{eq:content-score-app}
\end{equation}
where \(\mathcal{M}_c\) denotes the signals used for card type \(c\), and \(\lambda_{c,m}\) denotes their fixed weights.

For domain retrieval, the content signals include card-text semantic similarity, keyword overlap, and concept overlap. Card-text similarity captures overall instruction--card relevance, while keyword and concept matching provide lexical and structured compatibility. The resulting content score is combined with prototype similarity according to \eqref{eq:retrieval}, and the highest-scoring domain card is retained.

Format retrieval uses keyword overlap and answer-schema compatibility. Keyword matching captures explicit response cues in the instruction, while schema matching measures compatibility with the answer form represented by the format card. When an instruction explicitly specifies an answer schema, matching format cards are given priority.

After retrieving the domain and format cards, their information is used as additional cues for reasoning retrieval. Reasoning-card matching uses card-text semantic similarity, keyword overlap, domain compatibility, schema compatibility, and concept overlap. The domain and schema signals measure whether a candidate reasoning procedure is applicable to the retrieved context, while the remaining signals measure its semantic and lexical relevance to the instruction. The resulting content score is combined with prototype similarity according to \eqref{eq:retrieval}, and the highest-scoring reasoning card is retained.

\noindent\textbf{Stage-Specific Core Selection:}
For each activated capability \(s\), \method scores the stored stage-specific realizations \(\tau\in\mathcal{T}_s^t\) using four cues in \(\mathcal{G}(\mathbf q,\mathbf v)\). The instruction cue measures similarity between the current instruction and the stored instruction signature of a candidate realization; the metadata cue measures compatibility with the retrieved domain and format information; the image cue measures visual similarity to its stored visual signature; and the path-transition cue measures compatibility with the realization selected for the preceding capability.

Let \(h=(s,\tau)\) denote a stage-specific realization and \(h_{\mathrm{prev}}\) the realization selected for the preceding capability. For non-initial positions in a capability path, the transition cue is computed from the empirical transition frequency as:
\begin{equation}
    P_{\mathrm{trans}}
    \left(
        h\mid h_{\mathrm{prev}}
    \right)
    =
    \frac{
        N\!\left(
            h_{\mathrm{prev}}\!\rightarrow h
        \right)
    }{
        \sum_{h'}
        N\!\left(
            h_{\mathrm{prev}}\!\rightarrow h'
        \right)
    }.
    \label{eq:transition-cue}
\end{equation}
Any unavailable cue is omitted and the remaining coefficients are renormalized. The highest-scoring realization \(\widehat{\tau}_s\) is selected, and its corresponding per-layer cores \(C_{s,\widehat{\tau}_s}^{\ell}\) are used for capability execution.

\begin{table}[ht]
\centering
\small
\setlength{\tabcolsep}{5pt}
\caption{Fixed card-content matching and stage-specific core-selection coefficients used by \method. Signals are listed in the same order as their coefficients.}
\label{tab:routing-coefficients}
\resizebox{\columnwidth}{!}{%
\begin{tabular}{ccc}
\toprule
\textbf{Component} & \textbf{Signals} & \textbf{Coefficients} \\
\midrule
Domain card
& card-text semantic, keyword overlap, concept overlap
& 0.58,\,0.32,\,0.10 \\

Reasoning card
& card-text semantic, keyword overlap, domain match,
  schema match, concept overlap
& 0.48,\,0.27,\,0.10,\,0.10,\,0.05 \\

Format card
& keyword overlap, schema-rule match
& 0.60,\,0.40 \\

Core selection
& instruction cue, domain/format metadata, image cue,
  path transition
& 0.60,\,0.15,\,0.20,\,0.05 \\
\bottomrule
\end{tabular}%
}
\end{table}

\noindent\textbf{Coefficient Design:}
The coefficients are fixed globally rather than adapted to individual tasks or continual stages. Their relative magnitudes follow a simple evidence hierarchy based on how directly each signal matches the current instance to a knowledge card or stored realization. Motivated by prior hybrid-retrieval studies showing that semantic and lexical signals provide complementary relevance evidence~\citep{lee2023complementarity,kalra2025mor}, we use semantic matching as the primary retrieval signal while retaining lexical and structured cues for disambiguation. Prototype similarity receives the dominant interpolation weight \(\alpha=0.80\) because it directly compares the query with the source instruction clusters from which a card is constructed. Within card-content matching, semantic and keyword signals receive larger weights, while concept, domain, and schema signals provide auxiliary compatibility constraints. For core selection, instruction similarity serves as the primary task-free cue, visual similarity provides complementary instance-level evidence, and retrieved metadata and path-transition statistics provide additional context. This shared configuration reflects the heterogeneous instructions, response formats, and visual content encountered in MCIT, and its robustness to alternative weight configurations is evaluated on a held-out validation set in Section~\ref{app:coefficient-sensitivity}.

\subsection{Sensitivity to Retrieval and Routing Coefficients}
\label{app:coefficient-sensitivity}

We further examine the sensitivity of \method to the fixed retrieval and stage-specific core-selection coefficients on UCIT using a held-out validation set. Rather than varying individual coefficients exhaustively, we compare the default configuration with several representative alternatives while keeping the underlying retrieval and routing signals unchanged. For retrieval, we consider balanced weighting across the available evidence and a semantic-emphasized configuration; for core selection, we compare balanced cue weighting with increased emphasis on instruction similarity.

\begin{table}[ht]
\centering
\small
\renewcommand{\arraystretch}{1.08}
\setlength{\tabcolsep}{5pt}
\caption{Sensitivity to retrieval and stage-specific core-selection coefficients on UCIT. Weight tuples follow the signal order in Table~\ref{tab:routing-coefficients}, with \(\alpha=\alpha_{\mathrm{dom}}=\alpha_{\mathrm{rea}}\). Routing accuracy is computed over all routed capability positions.}
\label{tab:coefficient-sensitivity}
\resizebox{\columnwidth}{!}{%
\begin{tabular}{lccccc|cc}
\toprule
\textbf{Setting}
& \(\boldsymbol{\alpha}\)
& \textbf{Domain}
& \textbf{Reasoning}
& \textbf{Format}
& \textbf{Core Selection}
& \textbf{Routing Acc. (\%)}
& \(\boldsymbol{\overline{\mathcal{A}}}\)~$\uparrow$ \\
\midrule

\rowcolor{blue!6}
Default
& 0.80
& 0.58, 0.32, 0.10
& 0.48, 0.27, 0.10, 0.10, 0.05
& 0.60, 0.40
& 0.60, 0.15, 0.20, 0.05
& 99.304
& 70.67 \\

Balanced retrieval
& 0.50
& 0.33, 0.33, 0.34
& 0.20, 0.20, 0.20, 0.20, 0.20
& 0.50, 0.50
& default
& 66.354
& 67.41 \\

Semantic-emphasized retrieval
& 0.90
& 0.65, 0.25, 0.10
& 0.55, 0.20, 0.10, 0.10, 0.05
& default
& default
& 99.141
& 70.34 \\

Balanced routing
& default
& default
& default
& default
& 0.25, 0.25, 0.25, 0.25
& 99.761
& 70.51 \\

Instruction-emphasized routing
& default
& default
& default
& default
& 0.70, 0.10, 0.15, 0.05
& 98.982
& 70.22 \\
\bottomrule
\end{tabular}%
}
\end{table}

Routing accuracy is computed over all routed capability positions. For each position, we compare the stage-specific realization selected by the router with the reference realization associated with the corresponding task and capability. A position is counted as correct only when the two realization identifiers match. If either realization is unavailable, the position is marked as \textit{unknown} and treated as an error. The reported value is the micro accuracy over all routed positions, with \textit{unknown} positions included in the denominator. Task identity is used only to construct this evaluation reference and is never provided to the router.

As shown in Table~\ref{tab:coefficient-sensitivity}, \method remains stable under most variations of the retrieval and core-selection coefficients, with high routing accuracy and only minor changes in downstream performance. The main exception is Balanced retrieval, which assigns approximately equal importance to the available retrieval signals and leads to lower routing accuracy and average performance. Further inspection shows that this degradation is largely associated with \textit{unknown} positions, where the retrieved reasoning path contains a capability for which no corresponding stage-specific realization is available under the evaluation reference. This suggests that the primary failure arises from changes in capability-path retrieval rather than from stage-specific core selection itself. Overall, these results indicate that \method is robust to reasonable coefficient variations, while appropriate weighting of retrieval signals remains important for reliable capability-path construction.

\section{Knowledge Interface and Stage-Wise Execution}
\label{app:interface}

\subsection{Stored Fields and Forward Consumers}

Table~\ref{tab:card-interface} summarizes the typed interface. Only Domain
Guidance modifies the language-model input; reasoning and format records remain
structured control state.

\begin{table}[ht]
\centering
\small
\setlength{\tabcolsep}{4pt}
\caption{Knowledge-card fields and their consumers.}
\label{tab:card-interface}
\resizebox{\columnwidth}{!}{%
\begin{tabular}{cccc}
\toprule
Card & Principal stored fields & Forward consumer & In prompt? \\
\midrule
Domain
& concepts, keywords, knowledge, \textit{prompt guidance}
& generator
& prompt guidance only \\
Reasoning
& domains, schemas, ordered operations
& capability-path constructor
& no \\
Format
& schema, template, validation rule
& reasoning retrieval, core router
& no \\
\bottomrule
\end{tabular}%
}
\end{table}

For nonempty domain guidance \(g_{\mathrm{dom}}\), the augmented instruction is
\begin{equation}
    \widetilde{\mathbf q}=\texttt{[GUIDANCE: }g_{\mathrm{dom}}\texttt{]}\;\mathbf q.
    \label{eq:prompt-serialization}
\end{equation}
Retrieval uses the original \(\mathbf q\), and the prefix contains neither a
sample answer nor a dataset identifier. Removing it therefore leaves retrieval
and routing unchanged. An explicit format request in \(\mathbf q\) is not
duplicated by a retrieved template.

\subsection{Training and Inference Order}

At stage \(t\), \method (1) adds the validated shard to
\(\mathcal{K}^{1:t}\); (2) retrieves and canonicalizes capability paths;
(3) re-expresses recurring adapters, protects high-energy prefixes, expands
their rank, and initializes current cores; (4) optimizes current cores and
unprotected basis slices; and (5) stores multi-signal core-selection
prototypes.

After stage \(t\), inference uses the corresponding cumulative collection
\(\mathcal{K}^{1:t}\). It retrieves the three card types, optionally prepends
Domain Guidance, maps the Reasoning Path to capabilities, selects a
historical core independently for every capability, and activates the selected
updates in path order before autoregressive generation. Neither this procedure
nor the core score in \eqref{eq:core-routing} receives a task label.

For completeness, the stage-wise training objective and parameter masks used by
the algorithms below are
\begin{equation}
    \mathcal{L}_t
    =-
    \sum_{(\mathbf v,\mathbf q,\mathbf y)\in\mathcal{D}_t}
    \log p_{\Theta,\mathcal{U}^{1:t}_{\pi(\mathbf q)}}
    \big(\mathbf y\mid\mathbf v,\widetilde{\mathbf q}\big),
    \label{eq:objective}
\end{equation}
where \(\Theta\) is the frozen backbone and
\(\mathcal{U}^{1:t}_{\pi(\mathbf q)}\) contains the factorized updates
activated by the ordered path. For a recurring capability, the masks are
\begin{equation}
\begin{aligned}
    &C_{c,t}[1{:}K^{+},1{:}K^{+}] &&\text{trainable},\\
    &U_c[:,p_U{+}1{:}K^{+}],\quad
      V_c[p_V{+}1{:}K^{+},:] &&\text{trainable},\\
    &\{C_{c,\tau}\}_{\tau<t},\quad U_c[:,1{:}p_U],\quad
      V_c[1{:}p_V,:] &&\text{frozen}.
\end{aligned}
\label{eq:trainable-slices}
\end{equation}
Only capabilities activated by the current path receive gradients; inactive
capability realizations remain unchanged.

\section{Theoretical Properties of Adaptive Subspace Recycling}
\label{app:derivations}

We establish two properties of adaptive subspace recycling. First, we show that the QR reparameterization, energy-based rotation, and zero-padded rank expansion preserve every historical capability parameterization exactly before subsequent optimization. Second, we show that the protection rule bounds the historical energy exposed to the recyclable subspace. Throughout this section, we follow the notation of Section~\ref{sec:subspace-recycling}. For a fixed capability \(s\) and layer \(\ell\), we omit the capability and layer superscripts on the shared bases \(U_s^\ell\) and \(V_s^\ell\) for clarity. Here, \(\tau\in\mathcal{T}_s^{t-1}\) indexes the previous stages in which capability \(s\) has a stored core.

\subsection{Exact Preservation under Reparameterization}

Low-rank factorizations are not coordinate-unique. For any invertible matrices \(G_U,G_V\in\mathbb{R}^{K_s\times K_s}\):
\begin{equation}
    U C_{s,\tau} V
    =
    (U G_U)
    \left(
        G_U^{-1} C_{s,\tau} G_V^{-1}
    \right)
    (G_V V).
    \label{eq:factorization-gauge}
\end{equation}
Thus, directly protecting individual coordinates of the original bases would depend on the particular factorization. The QR reparameterization and energy-based rotation in \eqref{eq:qr}--\eqref{eq:rotation} instead establish an orthonormal, energy-ordered coordinate system without changing the represented capability parameterization.

\begin{proposition}[Exact preservation under reparameterization]
\label{prop:exact-preservation}
For every historical realization \(\tau\in\mathcal{T}_s^{t-1}\), the QR reparameterization and energy-based rotation preserve the capability parameterization exactly:
\begin{equation}
    \widetilde{U}
    \widetilde{C}_{s,\tau}
    \widetilde{V}
    =
    U C_{s,\tau} V
    =
    A_{s,\tau}^{\ell}.
    \label{eq:exact-preservation}
\end{equation}
Appending new basis directions and zero-padding the historical core also leave \(A_{s,\tau}^{\ell}\) unchanged.
\end{proposition}

\begin{proof}
Using \eqref{eq:qr}--\eqref{eq:rotation} and the orthogonality of \(P_U\) and \(P_V\):
\begin{align}
    \widetilde{U}
    \widetilde{C}_{s,\tau}
    \widetilde{V}
    &=
    Q_U P_U
    \left(
        P_U^\top B_{s,\tau} P_V
    \right)
    P_V^\top Q_V^\top \\
    &=
    Q_U B_{s,\tau} Q_V^\top \\
    &=
    Q_U R_U
    C_{s,\tau}
    R_V^\top Q_V^\top \\
    &=
    U C_{s,\tau} V
    =
    A_{s,\tau}^{\ell}.
\end{align}
Hence, the QR reparameterization and energy-based rotation change only the coordinate representation.

For the expansion from \(K^-\) to \(K^+\), define:
\begin{equation}
    U^{+}
    =
    \left[
        \widetilde{U}\;\;U_{\mathrm{new}}
    \right],
    \qquad
    V^{+}
    =
    \begin{bmatrix}
        \widetilde{V}\\
        V_{\mathrm{new}}
    \end{bmatrix},
    \qquad
    C_{s,\tau}^{+}
    =
    \begin{bmatrix}
        \widetilde{C}_{s,\tau} & 0\\
        0 & 0
    \end{bmatrix}.
    \label{eq:expanded-parameterization}
\end{equation}
Then:
\begin{equation}
    U^{+}C_{s,\tau}^{+}V^{+}
    =
    \widetilde{U}
    \widetilde{C}_{s,\tau}
    \widetilde{V}
    =
    A_{s,\tau}^{\ell}.
    \label{eq:rank-expansion-preservation}
\end{equation}
Therefore, both coordinate reparameterization and zero-padded rank expansion preserve every historical capability parameterization exactly.
\end{proof}

\subsection{Residual-Energy Bound}

The protection rule in \eqref{eq:protected-rank} selects the smallest prefixes that retain at least a fraction \(\rho\) of the accumulated historical energy. Define the corresponding projection matrices:
\begin{equation}
    \Pi_U
    =
    P_U[:,1{:}p_U]P_U[:,1{:}p_U]^\top,
    \qquad
    \Pi_V
    =
    P_V[:,1{:}p_V]P_V[:,1{:}p_V]^\top.
    \label{eq:protected-projectors}
\end{equation}

\begin{proposition}[Bounded residual historical energy]
\label{prop:residual-energy}
The aggregate historical energy outside the protected left and right subspaces is bounded by a fraction \(1-\rho\) of the corresponding total historical energy:
\begin{align}
    \sum_{\tau\in\mathcal{T}_s^{t-1}}
    \left\|
        (I-\Pi_U)B_{s,\tau}
    \right\|_F^2
    &\leq
    (1-\rho)
    \sum_{\tau\in\mathcal{T}_s^{t-1}}
    \|B_{s,\tau}\|_F^2,
    \label{eq:left-residual-bound}\\
    \sum_{\tau\in\mathcal{T}_s^{t-1}}
    \left\|
        B_{s,\tau}(I-\Pi_V)
    \right\|_F^2
    &\leq
    (1-\rho)
    \sum_{\tau\in\mathcal{T}_s^{t-1}}
    \|B_{s,\tau}\|_F^2.
    \label{eq:right-residual-bound}
\end{align}
\end{proposition}

\begin{proof}
From \eqref{eq:energy}:
\begin{equation}
    \operatorname{tr}(M_U)
    =
    \sum_{\tau\in\mathcal{T}_s^{t-1}}
    \|B_{s,\tau}\|_F^2
    =
    \sum_j\lambda_{U,j}.
\end{equation}
Since \(Q_U\) and \(Q_V\) have orthonormal columns:
\begin{equation}
    \|B_{s,\tau}\|_F
    =
    \|Q_U B_{s,\tau}Q_V^\top\|_F
    =
    \|A_{s,\tau}^{\ell}\|_F,
    \label{eq:parameterization-energy-equivalence}
\end{equation}
so the energy measured in the QR coordinates equals that of the corresponding historical capability parameterization.

For the left subspace:
\begin{align}
    \sum_{\tau\in\mathcal{T}_s^{t-1}}
    \left\|
        (I-\Pi_U)B_{s,\tau}
    \right\|_F^2
    &=
    \operatorname{tr}
    \left[
        (I-\Pi_U)M_U
    \right] \\
    &=
    \sum_{j>p_U}\lambda_{U,j}.
\end{align}
By the definition of \(p_U\) in \eqref{eq:protected-rank}:
\begin{equation}
    \sum_{j>p_U}\lambda_{U,j}
    \leq
    (1-\rho)
    \sum_j\lambda_{U,j},
\end{equation}
which proves \eqref{eq:left-residual-bound}. Applying the same argument to \(M_V\) proves \eqref{eq:right-residual-bound}.
\end{proof}

Proposition~\ref{prop:residual-energy} gives a direct interpretation of \(\rho\): each protected subspace captures at least a fraction \(\rho\) of the accumulated historical energy, leaving at most a fraction \(1-\rho\) exposed to recycling. Adaptive subspace recycling therefore preserves dominant historical directions while retaining low-energy directions as trainable capacity for subsequent adaptation.

\section{Sensitivity to Rank Expansion}
\label{app:rank-expansion}

\begin{table}[ht]
\centering
\small
\renewcommand{\arraystretch}{1.08}
\caption{Sensitivity to rank expansion on UCIT. \(K_{\mathrm{real}}^{\max}\) denotes the largest realized active rank, and Trainable Params reports the trainable parameter count under each setting. \(\Delta\) is relative to the default \(\delta=1\). Best task-wise and average results are bolded.}
\label{tab:rank-expansion}
\resizebox{\columnwidth}{!}{%
\begin{tabular}{ccccccccccc}
\toprule
$\boldsymbol{\delta}$
& $\boldsymbol{K_{\mathrm{real}}^{\max}}$
& \textbf{Trainable Params (M)}
& \textbf{ImgNet-R}
& \textbf{ArxivQA}
& \textbf{VizWiz}
& \textbf{IconQA}
& \textbf{CLEVR}
& \textbf{Flickr30k}
& \(\boldsymbol{\overline{\mathcal{A}}}\)~$\uparrow$
& $\boldsymbol{\Delta}$ \\
\midrule
0 & 8 & 22.176
& \textbf{85.87} & \textbf{93.00} & 60.70 & 65.53 & 59.27 & 55.97
& 70.06 & $-0.61$ \\

\rowcolor{blue!6}
1 (default) & 11 & 25.340
& 85.57 & 92.60 & 60.56 & 68.53 & 60.57 & 56.09
& 70.67 & 0.00 \\

2 & 14 & 28.491
& 85.77 & 92.80 & 60.23 & 69.27 & 61.87 & 56.40
& 71.06 & +0.39 \\

3 & 17 & 31.879
& 85.47 & 92.67 & 60.66 & 70.33 & 62.23 & \textbf{56.53}
& 71.32 & +0.65 \\

4 & 20 & 35.145
& 85.80 & 92.87 & \textbf{60.86} & \textbf{71.10}
& \textbf{62.47} & 56.38 & \textbf{71.58} & +0.91 \\
\bottomrule
\end{tabular}%
}
\end{table}

To examine the effect of rank expansion in adaptive subspace recycling, we vary the number of newly appended basis directions \(\delta\in\{0,1,2,3,4\}\) on UCIT using a held-out validation set. As shown in Table~\ref{tab:rank-expansion}, increasing \(\delta\) consistently improves average validation performance while increasing both the realized active rank and trainable parameter count. The gains are mainly observed on later tasks such as IconQA and CLEVR, while performance on earlier tasks remains relatively stable, suggesting that additional rank capacity primarily supports subsequent adaptation without substantially compromising previously learned capabilities. Although larger expansion budgets achieve further gains, they also incur higher parameter costs. We therefore use \(\delta=1\) as the default to balance continual adaptation performance and parameter efficiency.

\section{Computational Efficiency and Overhead}
\label{app:runtime}

\subsection{Training Efficiency}

We evaluate the computational efficiency of \method on UCIT.
For a focused comparison, we include LoRA-FT as a standard parameter-efficient
fine-tuning baseline, together with SAME, HiDe-LLaVA, and SEFE, which are the
three strongest baselines in terms of final average performance on UCIT.
We compare trainable parameter footprint and training overhead using
device-hours per task and per-device training throughput. The trainable
parameter count is measured at each continual stage and averaged over all
stages. Device-hours per task are computed as
\begin{equation}
    \mathrm{Device\!-\!h/task}
    =
    \frac{\sum_{t=1}^{T} N_t T_t^{\mathrm{train}}}{3600T},
    \label{eq:accelerator-hours}
\end{equation}
where \(T\) is the number of continual stages, \(N_t\) is the number of
devices used at stage \(t\), and \(T_t^{\mathrm{train}}\) is the corresponding
training time in seconds. Per-device training throughput is computed over the
complete continual-learning stream using the accumulated device time across
all stages.

\begin{table}[t]
\centering
\small
\renewcommand{\arraystretch}{1.08}
\setlength{\tabcolsep}{8pt}
\caption{Trainable-parameter footprint and training efficiency on UCIT. Lower is better for trainable-parameter count and device-hours, whereas higher is better for throughput.}
\label{tab:efficiency}
\resizebox{\columnwidth}{!}{%
\begin{tabular}{lccc}
\toprule
\textbf{Method} &
\textbf{Trainable Params (M)} $\downarrow$ &
\textbf{Device-h/task} $\downarrow$ &
\textbf{Train samples/s/dev.} $\uparrow$ \\
\midrule
SAME~\citep{xie2026same} & 73.286 & 7.12 & 1.391 \\
HiDe-LLaVA~\citep{guo2025hidellava} & 29.360 & 3.94 & 2.510 \\
SEFE~\citep{chen2025sefe} & 340.795 & 8.30 & 1.192 \\
LoRA-FT~\citep{hu2022lora} & 239.862 & {3.62} & {2.735} \\
\midrule
\rowcolor{blue!6}
\method (Ours) & {25.340} & 6.48 & 1.529 \\
\bottomrule
\end{tabular}
}%
\vspace{-3mm}
\end{table}

\begin{table}[ht]
\centering
\small
\renewcommand{\arraystretch}{1.08}
\caption{Inference efficiency on UCIT. Higher throughput and lower TTFT and
peak memory are better.}
\label{tab:inference-efficiency}
\resizebox{\columnwidth}{!}{%
\begin{tabular}{lcccc}
\toprule
\textbf{Method} & \textbf{Tokens/s}~$\uparrow$ & \textbf{TTFT (ms)}~$\downarrow$
& \textbf{Samples/s}~$\uparrow$ & \textbf{Peak memory (GB)}~$\downarrow$ \\
\midrule
Zeroshot & 11.4597 & 198.89 & 2.0963 & 14.23 \\
LoRA-FT~\citep{hu2022lora} & 6.2338 & 246.08 & 1.4529 & 15.48 \\
O-LoRA~\citep{wang2023orthogonal} & 8.8487 & 223.00 & 2.0312 & 28.33 \\
MoELoRA~\citep{chen2024coin} & 1.8979 & 619.01 & 0.4898 & 15.44 \\
ModalPrompt~\citep{zeng2025modalprompt} & 9.9822 & 269.88 & 1.7951 & 29.06 \\
CL-MoE~\citep{huai2025clmoe} & 2.4209 & 499.94 & 0.5511 & 14.83 \\
HiDe-LLaVA~\citep{guo2025hidellava} & 2.4283 & 594.72 & 0.6441 & 15.44 \\
SEFE~\citep{chen2025sefe} & 6.8335 & 230.57 & 1.3403 & 15.69 \\
SAME~\citep{xie2026same} & 3.0351 & 461.36 & 0.7373 & 29.43 \\
\midrule
\rowcolor{blue!6}
\method (Ours)& 5.3755 & 291.92 & 1.2086 & 14.73 \\
\bottomrule
\end{tabular}%
}
\end{table}

As shown in Table~\ref{tab:efficiency}, \method has the smallest trainable parameter footprint among all compared methods, optimizing only 25.340M parameters on average across continual stages. In particular, among the three strongest baselines on UCIT, \method uses fewer trainable parameters than HiDe-LLaVA, SAME, and SEFE. It also requires fewer device-hours and achieves higher training throughput than SAME and SEFE, while HiDe-LLaVA remains more efficient in terms of training time and throughput. LoRA-FT similarly trains faster, but does so with a substantially larger trainable parameter footprint. These results show that \method does not uniformly minimize every measure of training cost; rather, its main efficiency advantage lies in substantially reducing the number of optimized parameters while maintaining moderate training-time overhead.

\subsection{Inference Efficiency}

We compare inference efficiency with baselines on UCIT using token
throughput, time to first token (TTFT), sample throughput, and peak memory.
Token and sample throughput measure decoding and sample-level processing speed,
respectively, while TTFT reflects response latency and peak memory records
the maximum memory usage during inference. As shown in
Table~\ref{tab:inference-efficiency}, \method achieves 5.38 tokens/s and 1.21
samples/s, with a TTFT of 291.92\,ms and peak memory usage of 14.73\,GB.
Among the strongest-performing baselines, \method is more efficient than SAME
across all four metrics, while SEFE achieves higher throughput and lower
latency but requires more memory. \method also outperforms HiDe-LLaVA,
MoELoRA, and CL-MoE in both throughput and latency. Several other baselines,
including LoRA-FT, O-LoRA, and ModalPrompt, achieve higher throughput or lower
latency, but with higher peak memory usage. Notably, \method has the lowest
peak memory usage among all evaluated continual-learning methods. Overall,
these results show that \method maintains competitive inference efficiency
with low memory overhead.

\end{document}